\documentclass[conference]{IEEEtran}
\IEEEoverridecommandlockouts

\usepackage{cite}
\usepackage{amsmath,amssymb,amsfonts}
\usepackage{algorithmic}
\usepackage{graphicx}
\usepackage{textcomp}
\usepackage{xcolor}
\def\BibTeX{{\rm B\kern-.05em{\sc i\kern-.025em b}\kern-.08em
    T\kern-.1667em\lower.7ex\hbox{E}\kern-.125emX}}

\usepackage[ruled, vlined, linesnumbered]{algorithm2e}

\usepackage{amsthm}
\usepackage{subcaption}
\usepackage{multirow}
\usepackage{comment}
\usepackage{adjustbox}
\usepackage{array}

\newcommand{\CC}{\mathbb{C}}
\newcommand{\RR}{\mathbb{R}}
\newcommand{\PP}{\mathbb{P}}
\newcommand{\ZZ}{\mathbb{Z}}

\newcommand{\mm}{\mathfrak{m}}

\newcommand{\f}{\mathbf}

\newtheorem{definition}{Definition}
\newtheorem{notation}[definition]{Notation}
\newtheorem{remark}[definition]{Remark}

\newtheorem{lemma}[definition]{Lemma}
\newtheorem{proposition}[definition]{Proposition}

\begin{document}
\thispagestyle{plain}
\pagestyle{plain}

\title{Efficient closed-form approaches for pose estimation using Sylvester forms}

\author{Jana Vráblíková, Ezio Malis, Laurent Busé 
\thanks{The authors are with Centre Inria d'Université Côte d'Azur, France. \newline Email:{\tt\small \{first name, last name\}@inria.fr}}%
}

\maketitle

\begin{abstract}
Solving non-linear least-squares problem for pose estimation (rotation and translation) is often a time consuming yet fundamental problem in several real-time computer vision applications. With an adequate rotation parametrization, the optimization problem can be reduced to the solution of a~system of polynomial equations and solved in closed form.
Recent advances in efficient closed form solvers utilizing resultant matrices have shown a promising research direction to decrease the computation time while preserving the estimation accuracy. 
In this paper, we propose a new class of resultant-based solvers that exploit Sylvester forms to further reduce the complexity of the resolution. We demonstrate that our proposed methods are numerically as accurate as the state-of-the-art solvers, and outperform them in terms of computational time. 
We show that this approach can be applied for pose estimation in two different types of problems: estimating a pose from 3D to 3D correspondences, and estimating a pose from 3D points to 2D points correspondences. 
\end{abstract}


\section{Introduction}
\label{sec:intro}
Pose estimation from geometric correspondences is a longstanding problem in computer vision, with classical solutions dating back to early formulations of 3D registration and camera resectioning. 
A substantial line of research has focused on developing closed-form solvers for various pose estimation settings. Concerning camera resectioning (i.e. the problem of pose computation from 3D points to 2D points correspondences), numerous minimal solvers (the P3p problem) \cite{Gao-TPAMI-03,Kneip-CVPR-2011,Ke-CVPR-2017,Ding-CVPR-2023} and non-minimal solvers (the Pnp problem) \cite{Lepetit-IJCV-2008, Hesch-ICCV-2011, Zheng-ICCV-2013, Kneip-ECCV-2014}  have been proposed. Concerning 3D registration (i.e. the problem of pose computation from 3D to 3D correspondences), early solutions include the method proposed in \cite{Horn-JOSA-1987}, which provide exact solutions for noise-free point-to-point correspondences but do not naturally extend to point-to-line or point-to-plane constraints. More general treatments were later provided through unified optimization frameworks incorporating multiple geometric primitives, such as in \cite{Olsson-CVPR-2006}, which introduced optimal closed-form solutions based on polynomial formulations of the registration problem. Polynomial formulations enabling closed-form solutions have become particularly influential due to their robustness and predictability compared to direct iterative techniques \cite{Wientapper-CVIU-2018,Zhou-ICRA-2020, Malis2023,Malis2024}. These approaches often rely on Gröbner basis methods or resultant constructions to eliminate unknowns and recover camera pose.

Recently, resultant-based solvers have emerged as a powerful solution for constructing efficient closed-form solvers capable of handling mixed 3D to 3D correspondences \cite{Malis2023,Malis2024}. These method exploit the algebraic structure of the polynomial systems arising from quaternion-based parametrisations of rotation to solve a least square optimization problem imposing the unit norm constraint on the quaternion with a~Lagrangian. Using elimination matrices derived from multivariate resultants, one can obtain fast and accurate closed-form solvers. Despite their effectiveness, these solvers typically require working in high polynomial degrees. Since the larger is the polynomial degrees, the larger is the elimination matrix from which we obtain the solutions, obtaining resultant-based solvers with smaller degree will decrease computational cost.

Concurrently, advances in computational algebraic geometry have introduced more refined elimination techniques. In particular, the theory of Sylvester forms, initially introduced in \cite[\S 3.10]{Jouanolou2}, has been recently revisited and further generalized to the multigraded setting in \cite{Buse2022}, with a view towards applications to the solving of zero-dimensional polynomial systems. It provides new tools for constructing compact elimination matrices with lower algebraic degrees, in comparison with the classical Macaulay elimination matrices. These techniques have shown promising results for reducing the computational cost of polynomial solvers while preserving their algebraic completeness and numerical stability.

The main contribution of our work is to integrate Sylvester forms with the hidden-variable formulation of the resultant in order to obtain new resultant-based methods that operate in degrees 7 and 8, significantly reducing the size of the elimination matrices compared to the degree 9 approach proposed in \cite{Malis2024}. We give the theoretical foundations of our approach, relying on the concept of saturation of an ideal, and prove its validity. More specifically, other key contributions of this paper are (i) a detailed analysis of the rank of certain linear systems which allows us to prove the existence of our new elimination matrices (see Proposition \ref{prop:199}), and which also explains properties stated in \cite{Malis2024} (see Remark \ref{rem:E9}), (ii) a construction  of Sylvester forms tailored to our setting, providing structural results on their coefficients that ease their evaluation (see Lemma \ref{lem:brackets}).   

To our knowledge, this is the first application of Sylvester forms to a large variety of pose estimation problems, and the first demonstration that such forms can be used to derive faster, more compact closed-form solvers without sacrificing accuracy. This establishes a new connection between advanced elimination theory and practical computer vision algorithms.

\section{Theoretical background}
\label{sec:theoretical_background}

\subsection{Pose estimation from 3D to 3D correspondences}
\label{sse:pose_from_3D_to_3D_correspondences}
The registration of two sets of 3D points is typically formulated as a nonlinear optimisation problem after matching points to points, points to planes or points to lines. The objective of the problem is to estimate the pose (rotation matrix $\f R$ and translation vector $\f t$) from measured points and corresponding points, lines and planes. 
In this section, we briefly review a~unified formulation for those problems introduced in \cite{Olsson-CVPR-2006}. 

Let $\f m_c \in \RR^3$ be a point in a current frame $\mathcal F_c$. The current point $\f m_c$ is obtained from a reference point $\f m_r$ in a referent frame $\mathcal F_r$ as follows:
\begin{equation}\label{eq:14}
	\f m_c = \f R \ \f m_r + \f t 
\end{equation}
where $\f R \in \mathbb{SO}(3)$ is a $3 \times 3$ rotation matrix and $\f t \in \RR^3$ is a~translation vector. 

Stacking the 3 rows $\f r^\top_{k} = [r_{k1}, r_{k2}, r_{k3}]$ ($k \in \{1,2,3\}$) of $\f R$ into a $9 \times 1$ vector
\begin{equation}
	\f r = [\f R]_\lor = [\f r_{1}; \f r_{2}; \f r_{3}] 
\end{equation}
and introducing a $3 \times 9$ matrix 
\begin{equation}
	\f M = \left[\begin{array}{ccc}
		\f m_r^\top & \f 0 & \f 0 \\
		\f 0 & \f m_r^\top & \f 0 \\
		\f 0 & \f 0 & \f m_r^\top 
	\end{array} \right] 
\end{equation}
we can write \eqref{eq:14} as
\begin{equation}
	\f m_c = \f M \ \f r + \f t 
\end{equation}
The point $\f m_c$ correspond to a point $\f m = \f m_r$ in the current frame $\mathcal F_c$ if 
\begin{equation}
	\f m_c - \f m = \f M \ \f r + \f t - \f m = 0 
\end{equation}
The point $\f m_c$ lies on a line in the current frame $\mathcal F_c$ with a unit direction vector $\f d \in \RR^3$ that passes through a point $\f m \in \RR^3$ if
\begin{equation}
	[\f d]_\times^2 (\f m_c - \f m) = [\f d]_\times^2 (\f M \ \f r + \f t - \f m) 
\end{equation}
where 
\begin{equation}
	[\f v]_\times = \left[\begin{array}{ccc}
		0 & -v_3 & v_2 \\
		v_3 & 0 & -v_1 \\
		-v_2 & v_1 & 0
	\end{array} \right]
\end{equation}
for any vector $\f v = [v_1,v_2,v_3]^\top \in \RR^3$.
Finally, the point $\f m_c$ belongs to a plane in the current frame $\mathcal F_c$ with unit normal vector $\f n \in \RR^3$ that passes though a point $\f m \in \RR^3$ if
\begin{equation}
	\f n^\top(\f m_c - \f m) = \f n^\top(\f M \ \f r + \f t - \f m) = 0 
\end{equation}

Given $n_m$ point to point correspondences, $n_l$ point to line correspondences and $n_p$ point to plane correspondences, the optimisation problem can be written as
\begin{equation}\label{eq:34}
	\min_{\f r, \f t} \frac 1 2 \sum_{i=1}^{n_m} w_{m_i}^2 d_{m_i}^2 + \frac 1 2 \sum_{j=1}^{n_l} w_{l_j}^2 d_{l_j}^2 + \frac 1 2 \sum_{k=1}^{n_p} w_{p_k}^2 d_{p_k}^2  
\end{equation}
where $w_m, w_l$ and $w_p$ are weights and $d_m, d_l$ and $d_p$ are point to point, point to line and point to plane distances, respectively. 
Using a semi-definite weighting matrix $\f W \in \RR^{3 \times 3}$, we can write the square distances as follows
\begin{equation}
	d^2 = (\f M \ \f r + \f t - \f m)^\top \f W (\f M \ \f r + \f t - \f m) 
\end{equation}
where we select the matrix $\f W$ according to the type of correspondence,
\begin{itemize}
	\item $\f W = w_m^2 \f I$, for a point-to-point correspondence,
	\item $\f W = -w_l^2 [\f d]_{\times}^2$, for a point-to-line correspondence,
	\item $\f W = -w_p^2 \f n\f n^\top$, for a point-to-plane correspondence.
\end{itemize}
The optimisation problem \eqref{eq:34} can be therefore written as
\begin{equation}
\label{eqn:ls_3D_to_3D_correspondences}
	\min_{\f r, \f t} \sum_{i = 1}^{n} (\f M_i \ \f r + \f t - \f m_i)^\top \f W_i (\f M_i \ \f r + \f t - \f m_i) 
\end{equation}
where $n = n_m+n_l+n_p$.

\subsection{Pose estimation from 3D to 2D correspondences}
\label{sse:pose_from_3D_to_2D_correspondences}
Another classical problem in computer vision is the estimation of the poste from the projection of 3D points into the image, the Pnp problem:
\begin{equation}\label{eq:79}
Z_c \f q_c =  \f R \ \f m_r + \f t 
\end{equation}
where $Z_c = \f r^\top_3 \f m_r + t_3$. 
Given $n_q$ 3D point to 2D point correspondences, the weighted least squares optimisation problem can be written as:
\begin{equation}
\label{eqn:ls_3D_to_2D}
	\min_{\f r, \f t} \frac 1 2 \sum_{i=1}^{n_q} w_{q_i}^2 d_{q_i}^2
\end{equation}
where
\begin{equation}
\label{eqn:d_3D_to_2D}
d^2_{q} = (\f M \ \f r + \f t - Z_c \f q_c)^\top \f W (\f M \ \f r + \f t - Z_c \f q_c)
\end{equation}
where $\f W = w_q^2 \; \f I$, for a 3D point to 2D point correspondence. Since $Z_c \f q_c =  \f r^\top_3 \f m_r + \f t$, the optimisation problem \eqref{eqn:ls_3D_to_2D} can be written as
\begin{equation}
\label{eqn:ls_3D_to_2D_correspondences}
	\min_{\f r, \f t} \sum_{i = 1}^{n} (\f P_i \ \f r + \f Q_i \f t)^\top \f W_i (\f P_i \ \f r + \f Q_i \f t)
\end{equation}
where $n = n_q$ and:
\begin{equation}
\f P = \f M - \left[ \begin{array}{ccc} \f 0_{3x3} & \f 0_{3x3} & \f q_{c} \; \f m^\top_{r} \end{array} \right]
\end{equation}
\begin{equation}
\f Q = \f I - \left[ \begin{array}{ccc} \f 0_{3x1} & \f 0_{3x1} & -\f q_{c} \end{array} \right]
\end{equation}

\subsection{Reduction to a polynomial problem}

We parameterise the rotation by a unit quaternion:
\begin{equation}
	\f q = [q_r; \f q_i] = [w;x;y;z]
\end{equation}
The rotation matrix $\f R$ can be then parameterised as follows:
\begin{equation}
     \f R(\f q) = \f I + 2q_r[\f q_i]_{\times} + 2[\f q_i]_{\times}^2
\end{equation}
where $\f q^\top \f q = 1$. Therefore, the vector $\f r(\f q) = [\f R]_\lor$ is quadratic in the variables $w,x,y,z$. 

The translation vector $\f t$ can be eliminated from the equation \eqref{eqn:ls_3D_to_3D_correspondences} as shown in \cite{Olsson-CVPR-2006}. Similarly, it can be eliminated from the equation \eqref{eqn:ls_3D_to_2D_correspondences} as shown in \cite{ Kneip-ECCV-2014}. Therefore, we can solve a new equivalent problem that depends only on the four variables $w,x,y,z$. The new optimisation problem can be written in the following form:
\begin{equation} \label{eq:66}
	\f q = \mathrm{argmin}\{c(\f r(\f q))\} = \mathrm{argmin}\{\f r^\top \f A_r \f r + 2\f b_r^\top \f r + c_r \} 
\end{equation}
subject to
\begin{equation} \label{eq:70}
	\f q^\top \f q = 1
\end{equation}
where $\f A_r \in \RR^{9 \times 9}$ is a symmetric matrix, $\f b_r \in \RR^{9 \times 1}$ and $c_r \in \RR$. Note that for the Pnp problem we have $\f b_r = 0$ and $c_r =0$.


We can impose the constraint \eqref{eq:70} using the Lagrange multiplier method:
\begin{equation}\label{eq:133}
	\min_{\lambda, \f q}\mathcal L(\lambda,\f q) = c(\f r(\f q)) + \lambda (1 -\f q^\top \f q )
\end{equation}
The solutions of \eqref{eq:66} can be then obtained as solutions of the following polynomial system
\begin{align}
	\frac{\partial \mathcal L}{\partial \lambda} &= 1 -\f q^\top \f q  = 0 \label{eq:81}\\
	\frac{\partial \mathcal L}{\partial \f q} &= \f g^\top(\f q) - \lambda \f q^\top = 0 \label{eq:82}
\end{align}
where $\f g(\f q) = [g_w(\f q); g_x(\f q); g_y(\f q); g_z(\f q)]$ is the following $4 \times 1$ vector of degree 3 in $\f q$:
\begin{equation}
	\f g(\f q) = \left(\frac{\partial c(\f r)}{\partial \f r} \frac{\partial \f r(\f q)}{\partial \f q} \right)^\top 
\end{equation}
We note that $\f g(\f q)$ can be written as
\begin{equation}\label{eq:89}
    \f g(\f q) = \f g_3(\f q) + \f g_1(\f q) 
\end{equation}
where $\f g_3(\f q)$ is of homogeneous degree 3 in $\f q$ and $\f g_1(\f q)$ is linear in $\f q$. Substituting $\f q^\top \f q = 1$ into equation \eqref{eq:89} yields 4 polynomial equations which depend linearly on $\lambda$ and are homogeneous of degree 3 with respect to $\f q$:
\begin{align}\label{eq:93} \nonumber
    \f e(\f q,\lambda) & =\f g_3(\f q) + (\f q^\top \f q) \f g_1(\f q) -  \lambda(\f q^\top \f q)\f q  \\ 
					& = \hat{\f g}(\f q) -  \lambda(\f q^\top \f q)\f q = 0 
\end{align}
The equations $\f e(\f q, \lambda)$ also depend linearly on a coefficient vector 
\begin{equation}
	\f c = [c_{1,1},c_{1,2},\dots,c_{4,20}]^\top 
\end{equation}
namely the coefficients of $\hat{\f g}(\f q)$, and we sometimes write
\begin{equation}\label{eq:127}
	\f e(\f q,\lambda, \f c) = \f e(\f q,\lambda)
\end{equation}
to emphasise the coefficients.

The problem can be therefore reduced into the problem of finding real solutions of the polynomial system 
\begin{equation}\label{eq:e1-4}
	e_1(\f q, \lambda) = e_2(\f q, \lambda) = e_3(\f q, \lambda) = e_4(\f q, \lambda) = 0 	
\end{equation}
i.e., finding real points of the variety
$\mathrm V(I) \subseteq \PP^3 \times \mathbb A^1$
defined by the ideal 
$$I= (e_1,e_2,e_3,e_4) \subset \CC[\lambda][w,x,y,z]$$ 
which is graded with respect to the variables $\f q$. 

\subsection{Eliminating $\lambda$}\label{subsec:elimLambda}

It turns out that the projection of $V(I)$ on $\PP^3$, i.e.~the elimination of the parameter $\lambda$ from the equations \eqref{eq:93}, can be easily described. Indeed, taking  exterior product we get
\begin{equation}
	\hat{\f g}(\f q) \land \f q = 0
\end{equation}
which is equivalent to the condition 
\begin{equation}\label{eq:detdef}
    \mathrm{rank} \left( 
    \begin{array}{cccc}
    \hat{g}_w & \hat{g}_x & \hat{g}_y & \hat{g}_z \\
    w & x & y & z 
    \end{array} \right) < 2
\end{equation}
where   
$$[\hat{g}_w(\f q); \hat{g}_x(\f q); \hat{g}_y(\f q); \hat{g}_z(\f q)]=\hat{\f g}(\f q)$$ 
Therefore we obtain six polynomial equations, namely the $2\times 2$-minors of the above matrix, 
\begin{equation}\label{eq:172}
	f_i(\f q) = 0,\ i = 1,\dots,6 
\end{equation}
that are homogeneous of degree 4 in $\f q$. We notice that these equations depend linearly on the coefficients $\f c$ of the equations \eqref{eq:127}, and we write $\f f(\f q,\f c)$ to emphasise the coefficients.

Denote by $J$ the ideal of $\CC[w,x,y,z]$ generated by the polynomials $f_i(\f q)$, so that $V(J) \subset \PP^3$. For all $i=1,\ldots,6$ we have $f_i \in I$ (e.g.~$x\hat{g}_w - w\hat{g}_x = xe_1 - we_2$), so that any point in $V(I)$ yields a point in $V(J)$. Conversely, given a point in $V(J)$ such that ${\f q}^\top \f q \neq 0$, there is a unique $\lambda$ satisfying equations  \eqref{eq:93} at this point (observe that $\f q\neq 0$ at any point of $\PP^3$), hence a unique point in $V(I)$. 

\subsection{Closed-form solution via elimination matrices}\label{subsec:M9}
In \cite{Malis2024}, a method to find the solutions of equations \eqref{eq:93} based on the \emph{hidden variable} approach (see \cite[Chapter 3, §5]{Cox2005}) is proposed; in this section, we briefly review it. The variable $\lambda$ is considered as "hidden", that is to say that the polynomials \eqref{eq:82} are seen as polynomials in  $\f q$, $\lambda$ being interpreted as a parameter.  

\begin{notation}
For a given degree $d \in \mathbb N$, we denote the vector of homogeneous monomials of degree $d\in \mathbb{N}$ in $w,x,y,z$ by
\begin{equation*}
    \f m_d = [ w^{d_w} x^{d_x} y^{d_y} z^{d_z}: d_w+d_x+d_y+d_z = d]
\end{equation*}
The number of such monomials is $n_d = \binom{d+3}{d}$. 
\end{notation}

First, $4 \cdot n_6 = 4 \cdot 84 = 336$ equations constructed:
\begin{equation}\label{eq:E9}
    \f e(\f q, \lambda, \f c) \otimes \f m_6 = \f E_9(\f c, \lambda)\ \f m_9 = 0 
\end{equation}
where $\f E_9(\f c, \lambda)$ is a $336 \times 220$ coefficient matrix. For a general choice of $\lambda$, $\mathrm{rank}(\f E_9(\f c, \lambda)) = 220$.

Moreover, additional $6 \cdot n_5 = 6 \cdot 56 = 336$ equations that do not depend on $\lambda$ are considered:
\begin{equation}\label{eq:F9}
    \f f(\f q,\f c) \otimes \f m_5 = \f F_9(\f c)\ \f m_9 = 0
\end{equation}
where $\f F_9(\f c)$ is a $336 \times 220$ coefficient matrix and $\mathrm{rank}(\f F_9(\f c)) = 180$.

From matrices $\f E_9(\f c,\lambda)$ and $\f F_9(\f c)$, a $220 \times 220$ matrix $\f M_9$ admitting the following structure is obtained (see Remark \ref{rem:E9}):
{\small 
\begin{equation}\label{eq:30}
    \f M_9 = \left[\begin{array}{cc}
         \f A & \f B \\
         \f C & \f D 
    \end{array}\right] = 
    \left[\begin{array}{cc}
         \f A_0 & \f B_0 \\
         \f C_0 & \f D_0 
    \end{array}\right] - \lambda
    \left[\begin{array}{cc}
         \f A_1 & \f B_1 \\
         \f 0 & \f 0
    \end{array}\right]
\end{equation}
}

\noindent where the matrix $[\f A\  \f B]$ consists of a subset of rows of the matrix $\f E_9(\f c,\lambda)$, the matrix $[\f C\  \f D]$ consists of a subset of rows of the matrix $\f F_9(\f c)$. 
In particular, $\mathrm{size}(\f A) = 40 \times 40$, $\mathrm{size}(\f B) = 40 \times 180$, $\mathrm{size}(\f C) = 180 \times 40$ and $\mathrm{size}(\f D) = 180 \times 180$,
and furthermore, $\mathrm{rank}(\f M_9) = 220$, $\mathrm{rank}(\f A) = 40$, and $\mathrm{rank}(\f D) = 180$ for a general choice of $\lambda$. 

The solution of the pose estimation problem is then to find $\lambda$ such that
\begin{equation}\label{eq:48}
    \f M_9\ \f m_9 = 0 
\end{equation}
which is possible if
\begin{equation}
    \mathrm{det}(\f M_9) = 0 
\end{equation}
This leads to solving the generalised eigenvalue problem
\begin{equation}
    \mathrm{det}(\f Q_0 + \lambda \f Q_1) = 0 
\end{equation}
where 
\begin{equation}\label{eq:64}
    \f Q_i = \f A_i - \f B_i \f D_0^{-1} \f C_0, \ i=0,1
\end{equation}
%

\medskip

We notice that the matrix $\f M_9$ is built in such a way that its rows are filled by means of polynomial equations of degree 9 belonging to the ideal $I$; this is obvious for \eqref{eq:E9} and for \eqref{eq:F9} it follows from the fact that $f_i \in I$ for all $i=1,\ldots, 6$. In what follows we construct similar matrices $\f M_7$ and $\f M_8$ corresponding  to degrees 7 and 8 respectively. However, the extension to those degrees of the method described above is not straightforward: polynomial equations belonging to the ideal $I$ (and hence $J$) are not enough and it is necessary to introduce new equations.


\section{More efficient closed-form solutions \\ using Sylvester forms}

\subsection{Number of solutions}

The instances of the problem we are considering have finitely many (complex) solutions. In this section, we provide the number of points in $\PP^3_\CC$ defined by the ideal $J$ (see Section \ref{subsec:elimLambda}). It turns out that this counting is closely related to the analysis of the rank of matrices similar to \eqref{eq:F9}, built in arbitrary degrees. 

We recall that the ideal $J \subset \CC[w,x,y,z]$ is generated by the polynomials \eqref{eq:172} which are homogeneous polynomials. Therefore, $J$ is a graded ideal and we denote by $J_\nu$ its graded component of degree $\nu$, for all $\nu \in \ZZ$.  

Given an integer $d\geq 4$, we consider the matrix $\f F_{d}(c)$ built from $6\cdot n_{d-4}$ polynomials that form a basis of the vector space $(J)_{d}$:
\begin{equation}
	\f f(\f q,\f c) \otimes \f m_{d-4} = \f F_{d}(\f c)\ \f m_{d}^T
\end{equation}
The matrix $\f F_{d}(\f c)$ is a $(6 \cdot n_{d-4}) \times n_{d}$ coefficient matrix.

\begin{proposition}\label{prop:199}
	If $V(J) \subset \PP^3_\CC$ is finite, then it consists of 40 points counted with multiply. In addition,   
	$\mathrm{rank}(\f F_{d}(\f c)) = n_{d} - 40$ if and only if $d\geq 7$. 
\end{proposition}

\begin{proof} Set $R:=\CC[w,x,y,z]$. The proof relies on the analysis of the resolution of the quotient ring $R/J$ by graded free $R$-modules (we refer to \cite[Chapter 5 and 6]{Cox2005} for an introduction to these concepts). Since $V(J)$ is finite, the ideal $J$ is a determinantal ideal (i.e.~it is defined by minors of a matrix; see \eqref{eq:detdef}) which has maximal depth, here 3. As a~consequence, it admits the following free resolution (known as Eagon–Northcott resolution; see \cite[Theorem A.2.60]{Eis05}):
\begin{multline}\label{eq:ENres}
0 \rightarrow R(-10)\oplus R(-8)\oplus R(-6) \\
\rightarrow R(-7)^4\oplus R(-5)^4 \rightarrow R(-4)^6 \xrightarrow{\psi} R
\end{multline}	
where the notation $R(k)$, $k\in \ZZ$, denotes a shift in the grading: $R(k)_\nu=R_{k+\nu}$ for all integers $\nu$ and $k$. 

The map $\psi$ in \eqref{eq:ENres} is defined by the generators \eqref{eq:172} of $J$, which are of degree 4. Therefore, the transpose of $\f F_{d}(\f c)$ is a~matrix of the graded component
\begin{equation}
  \psi_d: R_{d-4} \rightarrow R_d  
\end{equation} 
of $\psi$ in suitable monomial bases (namely $\f m_d$ for the rows and $6 \f m_{d-4}$ for the columns). Now, the cokernel of $\psi_d$ is equal to the Hilbert function of $R/J$ in degree $d$ (see \cite[Chapter 1]{Eis05}). As $J$ is a defining ideal of points in $\PP^3_\CC$ (observe that $J$ is saturated because it has a free resolution of length 3), the Hilbert function of $R/J$ is equal to the Hilbert polynomial of $R/J$ if and only if $d$ is greater or equal to the Castelnuovo-Mumford regularity of $R/J$ (see \cite[Theorem 4.2]{Eis05}). In view of \eqref{eq:ENres}, the regularity of $R/J$ is equal to $10-3=7$. Moreover, the Hilbert polynomial of $R/J$ is a constant which is equal to the number of points in $\PP^3_\CC$, counted with multiplicity, defined by $J$. In our setting, it is equal to the quantity
\begin{equation}
    n_d-6n_{d-4}+(4n_{d-5}+4n_{d-7})-(n_{d-6}+n_{d-8}+n_{d-10})
\end{equation}
for all $d\geq 7$, which is equal to 40.
\end{proof}

\begin{remark}\label{rem:E9}
Proposition \ref{prop:199} shows that the matrix $\f M_9$ in \eqref{eq:30} is indeed of rank $180=220-40$. This is a key property in Section \ref{subsec:M9} because it proves the claimed structure \eqref{eq:30} of the matrix $\f M_9$. Indeed, since equations \eqref{eq:172} are contained in the ideal $I$, equations \eqref{eq:F9} can be found in equations \eqref{eq:E9} by linear operations on rows.
\end{remark}

%

\subsection{Saturation and Sylvester forms}\label{subsec:Syl}

\def\sat{\mathrm{sat}}
\def\mm{\mathfrak{m}}
\def\NN{\mathbb{N}}

The closed-form solution reviewed in Section \ref{subsec:M9} is based on the equations \eqref{eq:E9} which correspond to the graded component $I_9$ of degree 9 of the graded ideal $I\subset R_\lambda:=\CC[\lambda][w,x,y,z]$ (the grading is with respect to the four variables $w,x,y,z$; $\lambda$ being a parameter, it is considered to be of degree 0). To explain the choice of the degree 9, we need to introduce the ideal obtained from $I$ by saturation with respect to the ideal $\mm:=(w,x,y,z)$.
\begin{definition}\label{defIsat} The saturation of the graded ideal $I$ with respect to $\mm:=(w,x,y,z)$ is the ideal
\begin{equation}
	I^{\sat}:=\{ p\in R_\lambda \ \textrm{ such that } \ \exists n\in \NN : \mm^n p \subset I\}
\end{equation}
\end{definition}
Clearly $I\subset I^{\sat}$. Moreover $I$ and $I^{\sat}$ are equal after inversion in $R_\lambda$ of $w,x,y$ or $z$, which means that $V(I)$ and $V(I^\sat)$ are equal, including their local algebraic structures (e.g.~multiplicities). In particular $(I)_d=(I^{\sat})_d$ for $d \gg 0$.

\medskip

Now, suppose given an integer $d$ and consider the matrix $\f E_d(\f c, \lambda)$, built similarly to \eqref{eq:E9}. We expect two properties for this matrix in order to solve the polynomial system \eqref{eq:e1-4}:
\begin{itemize}
	\item it is of (full) rank $n_d$ for general values of $\f c$ and $\lambda$,
	\item it is not full rank for some given $\f c$ and $\lambda$ if and only if the corresponding polynomial system has solutions in $\PP^3_\CC$.
\end{itemize}
It is a known result in elimination theory that these two properties hold for any $d$ such that $(I)_d=(I^{\sat})_d$ (see e.g.~\cite[Theorem 3.20]{BCP23}). In addition, since $I$ is generated by 4 equations of degree 3 in $w,x,y,z$, this latter property holds if $d\geq 4(3-1)+1=9$ (see e.g.~\cite[Lemma 2.2]{Buse2022}).

\medskip

To overcome the limitation $d\geq 9$ to obtain matrices $\f E_d(\f c, \lambda)$ with the expected properties, it is necessary to introduce new equations. Our strategy is to take those equations in $I^\sat$ so that the solution set $V(I)$ is unchanged. In addition, since we are targeting closed-form solutions, those equations must be given in closed-form in the coefficients $\f c$. For that purpose, we will use Sylvester forms that have been initially introduced in \cite{Jouanolou2} (see also \cite[§2.10]{Buse2022}).

\medskip

Given $\alpha = (\alpha_1,\alpha_2,\alpha_3,\alpha_4) \in \NN^4$, we set  
$|\alpha| = \sum_{i=1}^4 \alpha_i$
and we define
\begin{equation}
    \f q^\alpha := [w^{\alpha_1+1},x^{\alpha_2+1},y^{\alpha_3+1},z^{\alpha_4+1}]
\end{equation}
Since the polynomials $e_1, e_2 ,e_3$ and $e_4$ are homogeneous of degree 3 with respect to $\f q$, for any $\alpha$ such that $|\alpha|<3$ it is possible to find decompositions 
\begin{equation}\label{eq:pij}
    e_i = [h_{i,1}, h_{i,2}, h_{i,3}, h_{i,4}]\ (\f q^\alpha)^T, \ i=1,\ldots, 4
\end{equation}
where $h_{i,j} = h_{i,j}(\f q,\lambda) \in \CC[\lambda][w,x,y,z]$ are homogeneous polynomials of degree $3-\alpha_j-1$ in $\f q$. 

\begin{definition}\label{def:SylF} For any $\alpha\in \NN^4$ such that $|\alpha|<3$, the determinant
	\begin{equation}
	    S_\alpha = \mathrm{det} \left(
	    \left[\begin{array}{cccc}
	        h_{1,1} & h_{1,2} & h_{1,3} & h_{1,4} \\
	        h_{2,1} & h_{2,2} & h_{2,3} & h_{2,4} \\
	        h_{3,1} & h_{3,2} & h_{3,3} & h_{3,4} \\
	        h_{4,1} & h_{4,2} & h_{4,3} & h_{4,4} 
	    \end{array}\right] \right)
	\end{equation}
	is called a Sylvester form (of $e_1,\ldots,e_4$ with respect to $\f q$).
\end{definition}

By construction, Sylvester forms depend on $\f q, \lambda$ and $\f c$. More specifically, $S_\alpha$ is homogeneous of degree $8 - |\alpha|$ in $\f q$, homogeneous of degree 4 in $\f c$ and is of degree at most 4 in $\lambda$ (notice that it is not homogeneous with respect to $\lambda$). Also, by construction Sylvester forms belong to $I^\sat$.

Although Sylvester forms depend on decompositions \eqref{eq:pij}, which are not unique, they are essentially unique in the following sense.

\begin{proposition}\label{prop:SylF} For any $\alpha$ such that $|\alpha|<3$, the class of $S_\alpha$ modulo $I$ is independent of the choices of decompositions \eqref{eq:pij}. Moreover, for any integer $\nu\in\{0,1,2\}$, the set 
\begin{equation}
    \{ S_\alpha \textrm{ such that } |\alpha|=\nu \}
\end{equation}
form a basis of the free $\CC[\lambda,\f c]$-module $(I^\sat/I)_{8-\nu}$. 
\end{proposition}
\begin{proof} We refer to \cite[Proposition 2.11]{Buse2022} and the references therein. 
\end{proof}

As a consequence of the above result, Sylvester forms can be added to the ideal $I$ to obtain a new ideal that have the same saturation but being itself saturated in a smaller degree. 
In what follows, we exploit this property to build new closed-form solutions to our problem.

\subsection{Closed-form solution in degree 8}

Consider the ideal $I'$ generated by the equations \eqref{eq:e1-4} and a~Sylvester from $S_0:=S_{(0,0,0,0)}$, i.e.~$I':=I+(S_0)$. According to Proposition \ref{prop:SylF}, 
\begin{equation}\label{eq:Ipsat}
(I')_d=(I^\sat)_d=(I')^\sat_d \ \textrm{ for all } d\geq 8
\end{equation}
Therefore, as explained in Section \ref{subsec:Syl}, the matrix $\f E_8'(\f c,\lambda)$ built from a basis of $(I')_8$ will have the expected properties. This matrix is actually the matrix  $\f E_8 (\f c,\lambda)$ to which an additional row corresponding to $S_0$ is added. 

In practice, a key ingredient is the choice made to compute $S_0$. Its degree with respect to $\lambda$ is of particular importance. To proceed, we consider a  decomposition of the polynomials $e_1,e_2,e_3,e_4$ (corresponding to rows) with respect to $\f q ^{(1,1,0,0)}=[w^2,x^2,y,z]$ such that:
{\small
\begin{equation}\label{eq:decS0}
\left(
\begin{array}{cccc}
p_{1,1}-\lambda w & p_{1,2} - \lambda w  & p_{1,3} - \lambda wy & p_{1,4} - \lambda wz \\
p_{2,1}-\lambda x & p_{2,2} - \lambda x  & p_{2,3} - \lambda xy & p_{2,4} - \lambda xz \\
p_{3,1}-\lambda y & p_{3,2} - \lambda y  & p_{3,3} - \lambda y^2 & p_{3,4} - \lambda yz \\	
p_{4,1}-\lambda z & p_{4,2} - \lambda z  & p_{4,3} - \lambda yz & p_{4,4} - \lambda z^2 
\end{array}
\right)
\end{equation}
}

\noindent where the $p_{i,j}$'s are polynomials in $\f q$ and $\f c$. The determinant of the above matrix is actually a Sylvester form $S_{(1,1,0,0)}$, which we denote by $S_{wx}$. Using classical rules of determinants, it appears that $S_{wx}$ is linear in $\lambda$. Moreover, a similar decomposition with respect to $\f q ^{(0,0,0,0)}=[w,x,y,z]$ is easily obtained from \eqref{eq:decS0} by multiplication of the first and second columns by $w$ and $x$ respectively. Therefore, by Definition \ref{def:SylF}, we may take $S_0:=wxS_{wx}$. It follows that $S_0$ can be written as
\begin{equation}\label{eq:pdk}
   p^{(8,4)}(\f q,\f c) - \lambda p^{(8,3)}(\f q, \f c)
\end{equation}
where $p^{(8,k)}$ is a polynomial of homogenous degree $8$ in $\f q$ and of homogenous degree $k$ in $\f c$, with $k = 3,4$. The following result is important in practice for the efficiency of the evaluation of $S_0$.

\begin{lemma}\label{lem:brackets} With the above notation, we write 
	\begin{equation}\label{eq:S_0}
	    p^{(8,k)} = \sum_{j=1}^{n_8} u_{j,k}\ m_j \   \textrm{ with } m_j \in \f m_8	
	\end{equation}	
and we assume that the decomposition \eqref{eq:decS0} is chosen such that for all $j=1,\ldots,4$, the polynomials $p_{1,j},\ldots,p_{4,j}$ have the same monomial supports (i.e.~the same set of monomials with nonzero coefficient). Then, the coefficients $u_{j,k}$ are $k\times k$ minors of the matrix 
\begin{equation}\label{eq:Cmatrix}
    \f C = \left[\begin{array}{cccc}
        c_{1,1} & c_{2,1} & \dots & c_{1,20} \\
        c_{2,1} & c_{2,2} & \dots & c_{2,20} \\
        c_{3,1} & c_{3,2} & \dots & c_{3,20} \\
        c_{4,1} & c_{4,2} & \dots & c_{4,20} 
    \end{array}\right]
\end{equation}
of the coefficients of the equations \eqref{eq:93}.
\end{lemma}
\begin{proof} We first assume that $\lambda=0$ in the polynomials $e_1,\ldots,e_4$ and hence in the decomposition \eqref{eq:decS0}. Consider the action of the special linear group $\mathrm{SL}(4,\CC)$ ($4\times 4$ matrices with determinant equal to 1) on the matrix $C$ given by matrix multiplication (on the left). By our assumptions on the choices of $p_{i,j}$, this action also corresponds to left multiplication on the matrix \eqref{eq:decS0}. Therefore, the determinant of this latter matrix is invariant under the action of $\mathrm{SL}(4,\CC)$. Applying the First Fundamental Theorem of Invariant Theory (see e.g.~\cite[Section 3.2]{Stu}), we deduce that it is a polynomial in the $4\times 4$-minors of \eqref{eq:Cmatrix} (these minors are themselves invariant under this action and actually generate the ring of invariants under this action). 
	
Now, the case where $\lambda$ is nonzero can be obtained from the previous case by substituting some coefficients $c_{i,j}$ by $c_{i,j}-\lambda$. We deduce that the determinant $S_{wx}$ is a polynomial of the $4\times 4$-minors of \eqref{eq:Cmatrix} after such substitutions. Expanding these minors, we obtain $4\times 4$-minors of the plain matrix \eqref{eq:Cmatrix}, then $\lambda$ multiplied by 
$3\times 3$-minors of the plain matrix \eqref{eq:Cmatrix} and terms depending on $\lambda^2$ up to $\lambda^4$. However, taking into account that $S_{wx}$ is linear in $\lambda$, all the terms that are not linear in, or independent of, $\lambda$ must cancel.
\end{proof}

We are now ready to state a new method to find the solutions of the system of equations \eqref{eq:93}. From a basis of $(I)_8$ we get $4\cdot 56 = 224$ equations 
\begin{equation}
    \f e \otimes \f m_5 =\f E_8(\f c,\lambda)\ \f m_8^T = 0
\end{equation}
The matrix $\f E_8(\f c,\lambda)$ is of size $224 \times 165$. We add to it a~single row $\f S$ from \eqref{eq:S_0}, i.e.~such that $S_0=\f S \f m_8^T$, to get the matrix $\f E_8'(\f c,\lambda)$ of size $225 \times 165$. In this way, the rows of $\f E_8'$ correspond to a basis of $(I')_8$. From \eqref{eq:Ipsat}, we deduce that $\f E_8'$ has full rank 165 for a general choice of $\lambda$. 

From equations \eqref{eq:172}, we get the $6 \cdot 35 = 210$ equations
\begin{equation}
    \f f \otimes \f m_4 =\f F_8(\f c)\ \f m_8^T = 0 
\end{equation}
where $\mathrm{size}(\f F_8(\f c)) = 210 \times 165$. By Proposition \ref{prop:199}, $\mathrm{rank}(\f F_8(\f c)) = 125$. Moreover, we already noticed that the rows of $\f F_8$ can be obtained in $\f E_8$ by linear operations on rows. Therefore, we proved that there exists a matrix $\f M_8$ of the form
{\small
\begin{equation}
    \f M_8 = \left[\begin{array}{cc}
         \f A & \f B \\
         \f C & \f D 
    \end{array}\right] = 
    \left[\begin{array}{cc}
         \f A_0 & \f B_0 \\
         \f C_0 & \f D_0 
    \end{array}\right] - \lambda
    \left[\begin{array}{cc}
         \f A_1 & \f B_1 \\
         \f 0 & \f 0
    \end{array}\right]
\end{equation}
}

\noindent where the matrix $[\f A\  \f B]$ consists of 39 rows of the matrix $\f E_8(\f c,\lambda)$ and the row $\f S$. The matrix $[\f C\  \f D]$ consists of 125 rows of the matrix $\f F_8(\f c)$. 
In particular, $\mathrm{size}(\f A) = 40 \times 40$, $\mathrm{size}(\f B) = 40 \times 125$, $\mathrm{size}(\f C) = 125 \times 40$ and $\mathrm{size}(\f D) = 125 \times 125$. Furthermore, $\mathrm{rank}(\f M_8) = 165$, $\mathrm{rank}(\f A) = 40$, and $\mathrm{rank}(\f D) = 125$ for a general choice of $\lambda$. 

The solutions of the system \eqref{eq:93} can be obtained by solving the generalised eigenvalue problem, following the procedure described in \eqref{eq:48}-\eqref{eq:64}.

\subsection{Closed-form solution in degree 7}

Following the same strategy as in the previous section, we develop a closed-form solution by using Sylvester forms of degree $7$ in $\f q$. We choose four Sylvester forms $S_\alpha$, for all $\alpha$ such that $|\alpha|=1$, and consider the ideal $I{''}$ defined as the ideal $I$ to which these four Sylvester forms are added. According to Proposition \ref{prop:SylF}, 
\begin{equation}\label{eq:IIpsat}
(I{''})_d=(I^\sat)_d=(I{''})^\sat_d \ \textrm{ for all } d\geq 7
\end{equation}

In practice, to choose appropriately our four Sylvester forms, we first compute two Sylvester forms $S_{(1,0,1,0)}$ and $S_{(0,1,0,1)}$ of degree 6 using decompositions similar to \eqref{eq:decS0} and then we set 
\begin{multline}\label{eq:263}
S_w:=yS_{(1,0,1,0)}, \ S_x:=zS_{(0,1,0,1)}, \\ S_y:=wS_{(1,0,1,0)}, \ S_z:=xS_{(0,1,0,1)}	
\end{multline}


%

The polynomials $S_w, S_x, S_y, S_z$ are of homogenous degree 7 in $\f q$, homogeneous of degree 4 in the coefficients $\f c$ and linear in $\lambda$. Thus, we can write them similarly to \eqref{eq:pdk}; for instance
\begin{equation}\label{eq:pw}
    S_w(\f q,\lambda) = p^{(7,4)}_w(\f q,\f c) - \lambda p^{(7,3)}_w(\f q, \f c) 
\end{equation}
where $p^{(7,k)}_w$ is a homogeneous polynomial of degree $7$ in $\f q$ and homogenous of degree $k$ in $\f c$ for $k = 3,4$. Similar expressions hold for $S_x, S_y$ and $S_z$. Lemma \ref{lem:brackets} also applies here so that the coefficients of these polynomials are $k\times k$ minors of the coefficient matrix $\f C$, which improves the practical efficiency of their evaluation.  


\medskip

Putting all the above ingredients together, we obtain another closed-form solution to find the solutions of the system of equations \eqref{eq:93}. From a basis of $(I)_7$, we construct $4\cdot 35 = 140$ equations 
\begin{equation}
    \f e \otimes \f m_4 =\f E_7(\f c,\lambda)\ \f m_7^T = 0 
\end{equation}
The matrix $\f E_7(\f c,\lambda)$ is of size $140 \times 120$. We add to it four rows obtained with the coefficients of $S_w, S_x, S_y$ and $S_z$ by means of expressions \eqref{eq:pw}. We obtain the matrix $\f E_7^{''}(\f c,\lambda)$ of size $144 \times 120$ which has full rank 120 for a general choice of $\lambda$ (which is ensured by \eqref{eq:IIpsat}). 
Now, equations \eqref{eq:172} yields $6 \cdot 20 = 120$ equations
\begin{equation}
    \f f \otimes \f m_3 =\f F_7(\f c)\ \f m_7^T = 0 
\end{equation}
where $\mathrm{size}(\f F_7(\f c)) = 120 \times 120$ and $\mathrm{rank}(\f  F_7(\f c)) = 80$ by Proposition \ref{prop:199}. It follows that one can construct a matrix $\f M_7$ of the form
{\small
\begin{equation}
    \f M_7 = \left[\begin{array}{cc}
         \f A & \f B \\
         \f C & \f D 
    \end{array}\right] = 
    \left[\begin{array}{cc}
         \f A_0 & \f B_0 \\
         \f C_0 & \f D_0 
    \end{array}\right] - \lambda
    \left[\begin{array}{cc}
         \f A_1 & \f B_1 \\
         \f 0 & \f 0
    \end{array}\right]
\end{equation}
}

\noindent where the matrix $[\f A\  \f B]$ consists of a 36 rows of the matrix $\f E_7(\f c,\lambda)$ and 4 rows of the coefficients of the polynomials $S_w, S_x,S_y, S_z$. The matrix $[\f C\  \f D]$ consists of 80 rows of the matrix $\f F_7(\f c)$. 
In particular, $\mathrm{size}(\f A) = 40 \times 40$, $\mathrm{size}(\f B) = 40 \times 80$, $\mathrm{size}(\f C) = 80 \times 40$ and $\mathrm{size}(\f D) = 80 \times 80$. Furthermore, $\mathrm{rank}(\f M_7) = 120$, $\mathrm{rank}(\f A) = 40$, and $\mathrm{rank}(\f D) = 80$ for a~general choice of $\lambda$. 

The solutions of the system \eqref{eq:93} can be obtained by solving the generalised eigenvalue problem, following the procedure described in \eqref{eq:48}-\eqref{eq:64}.

\subsection{Constructing matrices $\textbf Q_0$ and $\textbf Q_1$}
To solve the generalized eigenvalue problem, we need to construct matrices $\textbf Q_0$ and $\textbf Q_1$, see \eqref{eq:64}. In \cite{Malis2024}, these matrices are obtained by computing Schur complements of the matrix $\textbf M_9$, which is composed of fixed subsets of rows of the coefficient matrices $\textbf E_9$ and $\textbf F_9$. This fast construction, however, can fail if the block $\textbf D$ of the matrix $\textbf M_9$ (see \eqref{eq:30}) is not invertible. Here, we propose a  procedure for the construction of the matrices $\textbf Q_0$ and $\textbf Q_1$, which outputs matrices with the best conditioning for a given choice of the monomial order. To unify the notation, here we set $\textbf E'_7 := \textbf E''_7$. 

\begin{algorithm}[h]
\small 
\caption{Optimized selection of rows and Schur complements}
\DontPrintSemicolon

\KwIn{Coefficient matrices $\mathbf{E}'_d$ and $\mathbf{F}_d$, monomials $\mathbf{m}_d$ ($d = 7,8$)}
\KwOut{Matrices $\mathbf{Q}_0$ and $\mathbf{Q}_1$}

\BlankLine
Fix a monomial order on $\mathbf{m}_d$ and reorder the columns of $\mathbf{E}'_d$ and $\mathbf{F}_d$ (the first 40 monomials index the columns of $\overline{\mathbf{A}_0}, \overline{\mathbf{A}_1}$ and $\overline{\mathbf{C}_0}$)\;

Define blocks such that $[\overline{\mathbf{A}_0} \ \overline{\mathbf{B}_0}] + \lambda [\overline{\mathbf{A}_1} \ \overline{\mathbf{B}_1}] = \mathbf{E}'_d$ and $[\overline{\mathbf{C}_0}\ \overline{\mathbf{D}_0}] = \mathbf{F}_d$\;

Compute the QR decomposition $\overline{\mathbf{D}_0} = \mathbf{Q}_D \mathbf{R}_D$\;

Set $\mathbf{C}_0 = \mathbf{Q}_D^T \overline{\mathbf{C}}$ and $\mathbf{D}_0 = \mathbf{Q}_D^T \overline{\mathbf{D}} = \mathbf{R}_D$\;

Compute Schur complements $\overline{\mathbf{Q}_i} = \overline{\mathbf{A}_i} - \overline{\mathbf{B}_i} \mathbf{D}_0^{-1} \mathbf{C}_0$ for $i = 0,1$\;

Compute the QR decomposition $\mathbf{Q}_Q \mathbf{R}_Q = \overline{\mathbf{Q}_1}$\;

Let $\mathbf{Q}_0 = \mathbf{Q}_Q^T \overline{\mathbf{Q}_0}$ and $\mathbf{Q}_1 = \mathbf{Q}_Q^T \overline{\mathbf{Q}_1} = \mathbf{R}_Q$\;
\end{algorithm}
\begin{remark}
    Both matrices $\overline{\mathbf{D}_0}$ and $\overline{\mathbf{Q}_1}$ are tall matrices, i.e., they have more rows than columns. The \emph{economical} QR decomposition (e.g. \emph{qr($\overline{\mathbf{D}_0}$,"econ")} in Matlab) outputs an orthonormal matrix $\mathbf Q$ and a square, upper triangular matrix $\textbf R$ with non-zero entries on the diagonal. The latter one can be inverted efficiently and is reused in the later steps of the procedure.
\end{remark}
\section{Simulations and experimental results}

We apply the methods to simulated and real data, comparing them in terms of accuracy and computational time with the state-of-the-art methods for the pose estimation problem from 3D to 3D as well as from 3D to 2D correspondences.

The solving process can be divided into three steps:
\begin{enumerate}
    \item Derivation of the problem in the "canonical" form \eqref{eq:66}. 
    \item Computation of the solutions of the polynomial system given by equations \eqref{eq:81} and \eqref{eq:82}.
    \item Selection of the minimal solution, i.e. of the solution for which \eqref{eq:66} is minimal.
\end{enumerate}
The computational time presented in the results only reflects the time necessary for the second step, because the times needed for the first and final steps are identical for all the presented methods.

\subsection{Pose estimation from 3D to 3D correspondences}

We compare the proposed methods with the state-of-the-art closed-form h-resultant based algorithm by Malis \cite{Malis2024} and the algorithms by Wientapper et al. \cite{Wientapper-CVIU-2018} and by Zhou, Wang, and Kaess \cite{Zhou-ICRA-2020} (without the Newton-Raphson iterations to refine the results).

\subsubsection{Experiments with simulated data}
The data are generated using the simulation described in \cite{Malis2024}. In particular, 3D points are randomly sampled on a 10m radius sphere. For lines and planes, unit direction vectors and normal vectors are generated randomly. Rotations are generated by uniformly sampling the Euler angles $\phi$ and $\theta$, where $\phi,\theta \in [0, 2\pi]$ and $\theta \in [0,\pi]$. The translation vectors are uniformly distributed within the range $[-10 \mathrm{m} ,10 \mathrm{m}]$.

For $n_m$ point-to-point, $n_l$ point-to-line, and $n_p$ point-to-plane correspondences, we denote $N = 3n_m + 2n_l + n_p$. Given $N$, a combination of point-to-point, point-to-line and point-to-plane correspondences is randomly generated.

We compare the estimated rotation $\hat{\f R}$ and translation $\hat{\f t}$ to the ground truth rotation $\f R$ and translation $\f t$. The rotation error $\delta r$ is the absolute value of the rotation angle computed as $\delta r = \|\mathrm{log}(\hat{\f R}\f R^T)\|_F$, where log denotes the logarithm of a rotation matrix and $\| \cdot \|_F$ denotes the Frobenius norm that gives the magnitude of the rotation angle. The translation error is $\delta t = \|\hat{\f t}- \f t\|$. The median computation time, the average rotation error $\mu(\delta r)$ and the average translation error $\mu(\delta t)$ are computed for 1000 trials.

\noindent \textbf{Accuracy comparison.}
In the first simulation, we do not add any noise to the simulated data in order to check the numerical sensitivity of the algorithms. Figure \ref{fig:noise0} presents the average rotation and average translation errors for an increasing number of correspondences. In  the absence of noise, the proposed methods in degree 7 and 8 are  numerically more stable than the degree 9 method, perform similarly to \cite{Wientapper-CVIU-2018} and significantly outperform \cite{Zhou-ICRA-2020}. 
\begin{figure}[t]
\captionsetup[subfigure]{labelformat=empty}
\centering
\begin{subfigure}[b]{0.5\linewidth}
\centering
\includegraphics[width=0.95\linewidth]{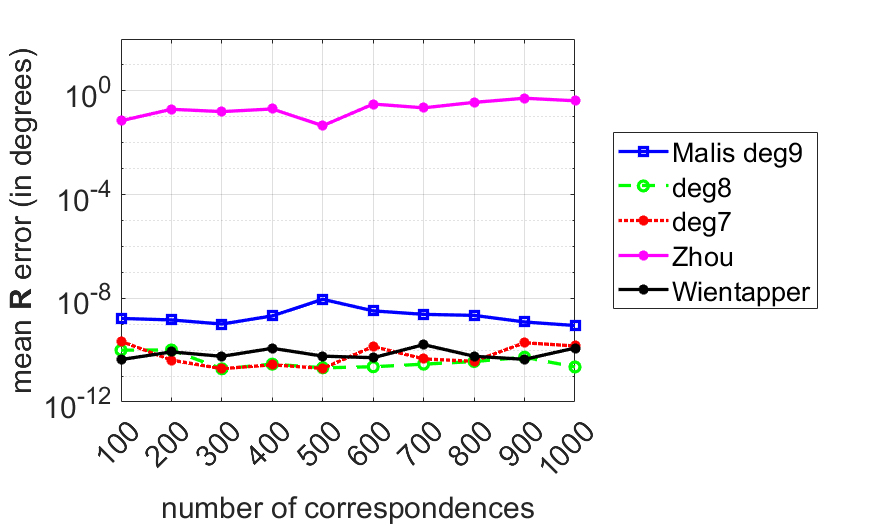}
\caption{Rotation error (degrees).}
\end{subfigure}%
\begin{subfigure}[b]{0.5\linewidth}
\centering
\includegraphics[width=0.95\linewidth]{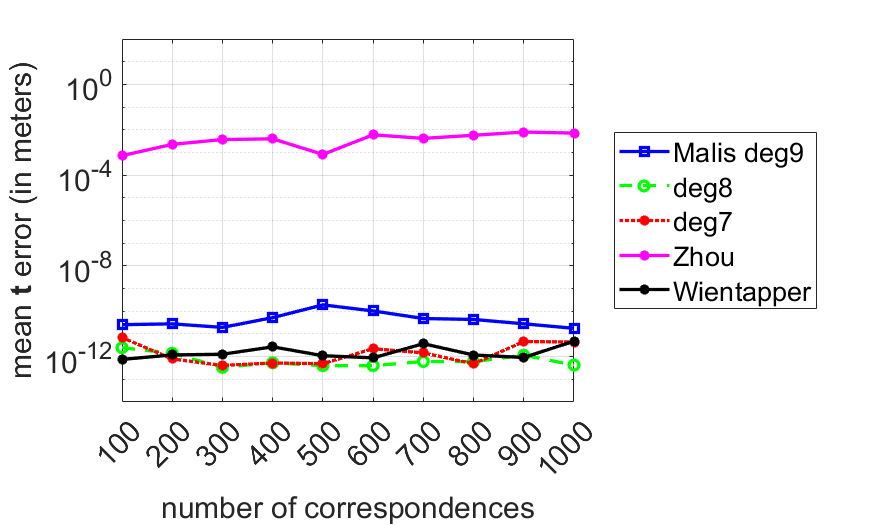}
\caption{Translation error (m).}
\end{subfigure}
\caption{Number of correspondences increases from 100 to 1000, the noise standard deviation is set to 0 m.}
\label{fig:noise0}
\end{figure}
Figure \ref{fig:noiseInc} shows the average rotation and translation error for $1000$ correspondences and an increasing level of noise. 
\begin{figure}[t]
\captionsetup[subfigure]{labelformat=empty}
\centering
\begin{subfigure}[b]{0.5\linewidth}
\centering
\includegraphics[width=0.75\linewidth]{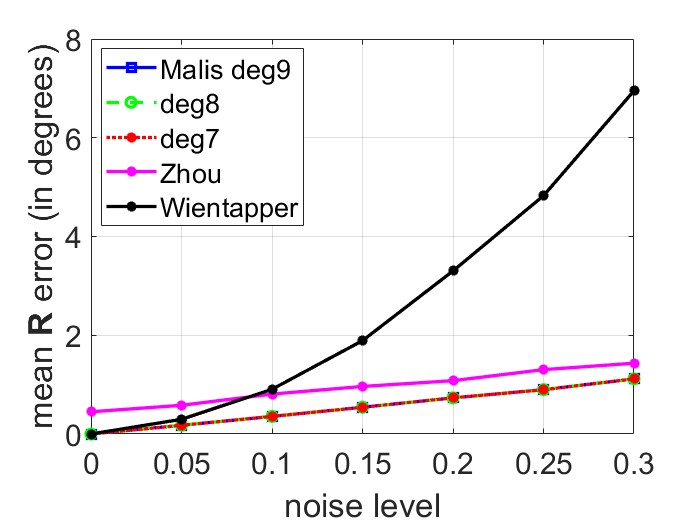}
\caption{Rotation error (degrees).}
\end{subfigure}%
\begin{subfigure}[b]{0.5\linewidth}
\centering
\includegraphics[width=0.75\linewidth]{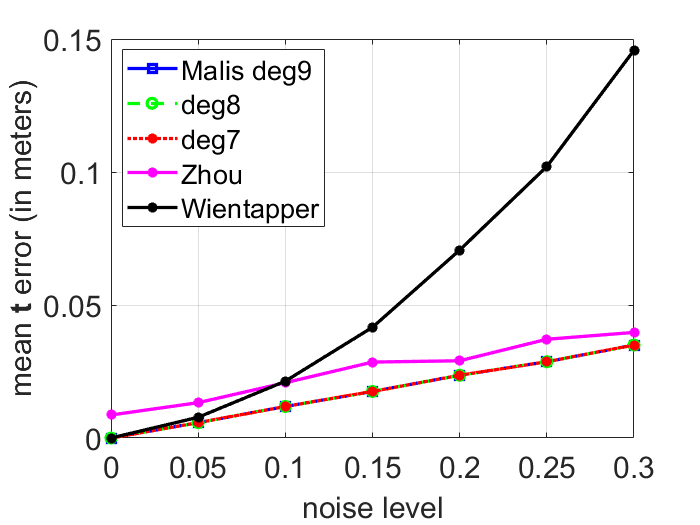}
\caption{Translation error (m).}
\end{subfigure}
\caption{Noise standard deviation increases from 0 to 0.3 m, the number of correspondences is set to 1000.}
\label{fig:noiseInc}
\end{figure}
Figure~\ref{fig:noise02} presents the average rotation and translation error for increasing number of correspondences. The noise standard deviation is set to $\sigma=0.2$m. In the simulations shown in Figures \ref{fig:noiseInc} and \ref{fig:noise02}, the proposed methods perform as well as the algorithm from \cite{Malis2024} and outperform methods \cite{Zhou-ICRA-2020} and \cite{Wientapper-CVIU-2018}.

\begin{figure}[t]
\captionsetup[subfigure]{labelformat=empty}
\centering
\begin{subfigure}[b]{0.5\linewidth}
\centering
\includegraphics[width=0.95\linewidth]{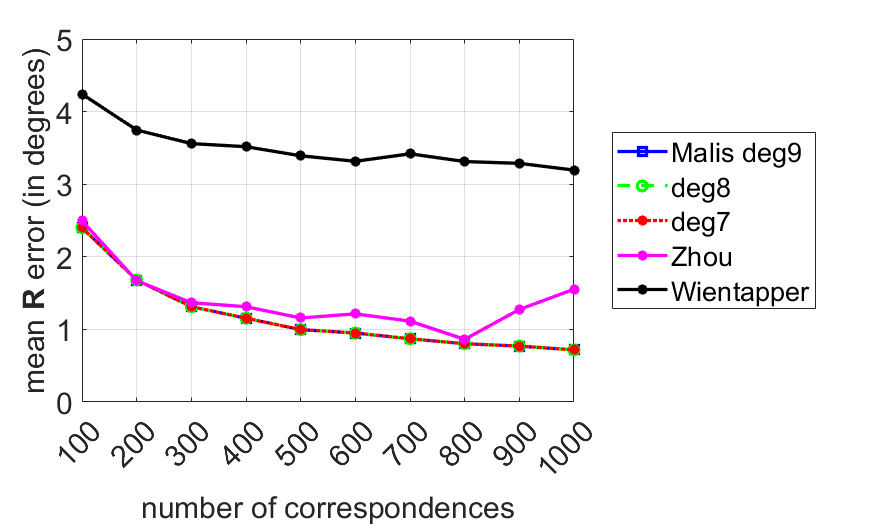}
\caption{Rotation error (degrees).}
\end{subfigure}%
\begin{subfigure}[b]{0.5\linewidth}
\centering
\includegraphics[width=0.95\linewidth]{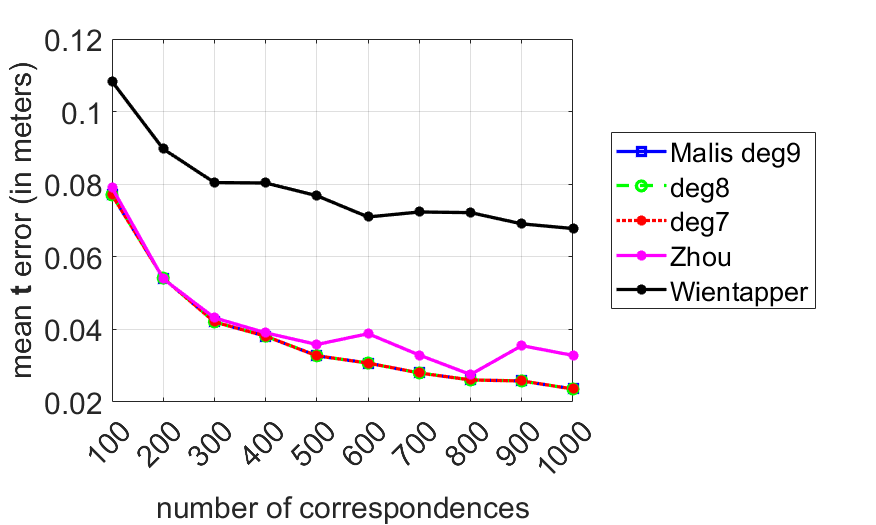}
\caption{Translation error (m).}
\end{subfigure}
\caption{Number of correspondences increases from 100 to 1000, the noise standard deviation is set to 0.2 m.}
\label{fig:noise02}
\end{figure}

\noindent \textbf{Computational time comparison.}
In the proposed methods, it is necessary to evaluate the coefficients of the polynomials $S_i$. 
Experimentally, we observed that there is no notable difference in terms of time or numerical stability for different choices of the Sylvester forms of degree 6 and the decompositions \eqref{eq:pij} of the polynomials \eqref{eq:172}. 

Figure \ref{fig:time} shows the median computational time for $N$ varying between 10 and 3000, with $\sigma = 0.2$m. For the sake of visibility, the results are separated into two figures. 
Our method with $d = 7$ outperforms the method by Malis \cite{Malis2024}, which uses a fixed subset of rows to construct the matrix $\textbf M_9$. Both proposed methods are faster that the method \cite{Zhou-ICRA-2020} and \cite{Wientapper-CVIU-2018}. We used a Matlab wrapper of the C++ implementation of the algorithm \cite{Wientapper-CVIU-2018} provided by the authors, and Matlab implementations of the other methods.

\begin{figure}[t]
\captionsetup[subfigure]{labelformat=empty}
\centering
\begin{subfigure}[b]{0.5\linewidth}
\centering
\includegraphics[width=0.95\linewidth]{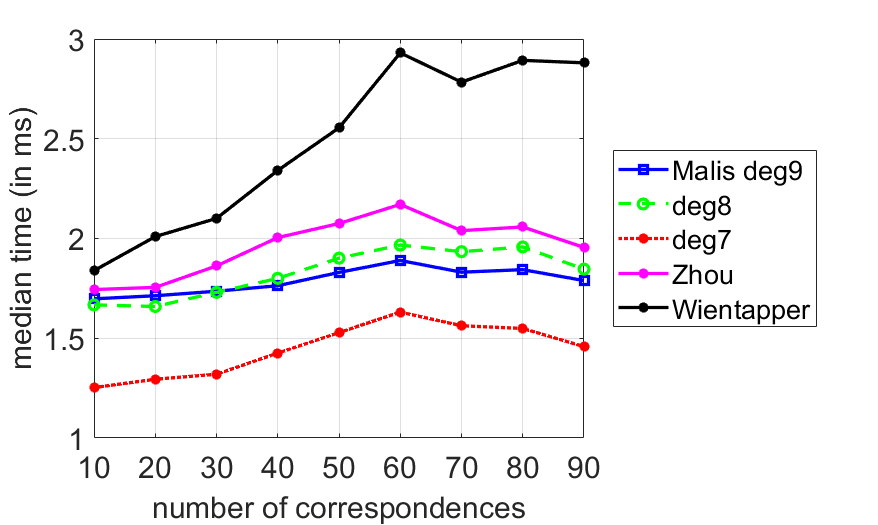}
\caption{$N$ from 10 to 90.}
\end{subfigure}%
\begin{subfigure}[b]{0.5\linewidth}
\centering
\includegraphics[width=0.95\linewidth]{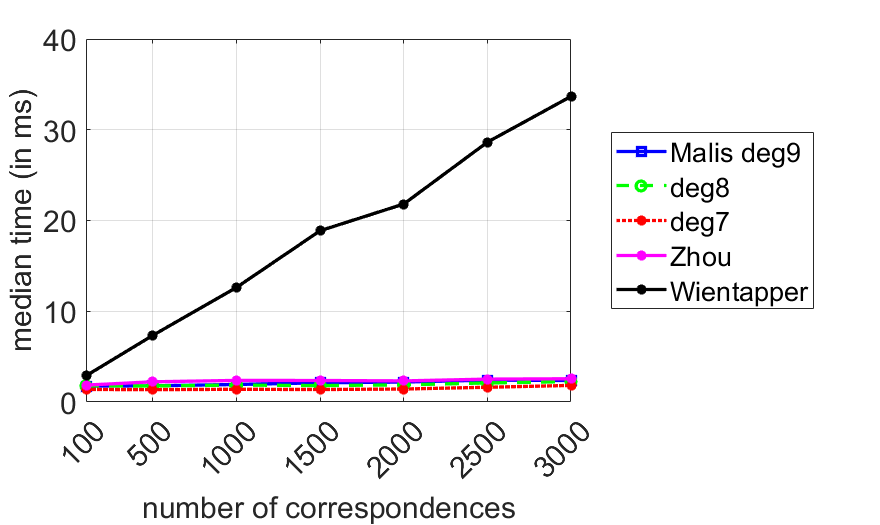}
\caption{$N$ from 100 to 3000.}
\end{subfigure}
\caption{Comparison of computational times.}
\label{fig:time}
\end{figure}

\subsubsection{Experiments with real data}
Similarly to the previous works \cite{Zhou-ICRA-2020,Malis2024}, the KITTI dataset \cite{Geiger2013IJRR} is used for the experimental evaluation. The current set of 3D points (LiDAR scan $k{+}1$) is segmented to extract planar structures and matched with the closest 3D points from the reference set (LiDAR scan $k$). Planes are detected in the point cloud by performing a~least-squares fitting on the $8$-nearest neighbors 
of each point. If the resulting least-squares residual is below a predefined threshold, the neighboring points are considered to belong to the same plane. Point-to-plane correspondences are then obtained using an Iterative Closest Point (ICP) strategy, where the estimated pose is ignored and only the correspondences are retained. Once the point-to-plane correspondences have been determined, the optimal pose between scans $k{+}1$ and $k$ is estimated and used to register the reference 3D points into the current frame. A robust Tukey M-estimator is employed to compute the weights of the weighted least-squares problem derived from \eqref{eqn:ls_3D_to_3D_correspondences}. Finally, the pose estimated with state-of-the-art approaches and the methods proposed in this paper are compared with the ground truth, and a translation error $\delta t$ (in meters) and a rotation error $\delta r$ (in degrees) are computed for each frame.

Table \ref{table1} shows the mean and standard deviation of translation and rotation errors and computation time obtained on sequences 03, 04, and 07 of the KITTI dataset. The proposed methods outperform the other methods in terms of both accuracy and computational time. 
\begin{table}[h]
\scriptsize
\begin{tabular}{|wl{0.17\linewidth}|wc{0.21\linewidth}|wc{0.21\linewidth}|wc{0.21\linewidth}|}
\hline
\multicolumn{4}{|c|}{KITTI sequence 03 (800 frames)} \\ \hline
\multirow{2}{*}{method} & rotation (°)  & translation (m) & time (ms) \\ 
 & $\mu(\delta r) \pm \sigma(\delta r)$ & $\mu(\delta t) \pm \sigma(\delta t)$ & $\mu(t) \pm \sigma(t)$ \\ \hline
Malis \cite{Malis2024}  &  0.2741$\pm$6.3519  &  0.0237$\pm$0.2125  &  3.3211$\pm$0.9909  \\ \hline
Wientapper \cite{Wientapper-CVIU-2018} & 0.0521$\pm$0.0365  &  0.0163$\pm$0.011  &  62.893$\pm$12.574  \\ \hline
Zhou \cite{Zhou-ICRA-2020}  & 0.0746$\pm$0.6813  &  0.018$\pm$0.0481  &  2.9928$\pm$0.9365  \\ \hline
deg8   & \textbf{0.0493}$\pm$\textbf{0.0162}  &  \textbf{0.0181}$\pm$\textbf{0.0115}  &  3.0317$\pm$0.744  \\ \hline
deg7   & \textbf{0.0493}$\pm$\textbf{0.0162}  &  \textbf{0.0181}$\pm$\textbf{0.0115}  &  \textbf{2.5318}$\pm$\textbf{0.4278}  \\ \hline
\end{tabular}

\vspace{1mm}
\begin{tabular}{|wl{0.17\linewidth}|wc{0.21\linewidth}|wc{0.21\linewidth}|wc{0.21\linewidth}|}
\hline
 \multicolumn{4}{|c|}{KITTI sequence 04 (270 frames)} \\ \hline
\multirow{2}{*}{method} & rotation (°)  & translation (m) & time (ms) \\ 
 & $\mu(\delta r) \pm \sigma(\delta r)$ & $\mu(\delta t) \pm \sigma(\delta t)$ & $\mu(t) \pm \sigma(t)$ \\ \hline
Malis \cite{Malis2024} & \textbf{0.0329}$\pm$\textbf{0.0191}  &  \textbf{0.0163}$\pm$\textbf{0.0093} &  3.3927$\pm$0.8963  \\ \hline
Wientapper \cite{Wientapper-CVIU-2018}& 0.0373$\pm$0.0274  &  0.0165$\pm$0.0094 & 66.125$\pm$13.96   \\ \hline
Zhou  \cite{Zhou-ICRA-2020} & 0.5892$\pm$9.1232 &  0.0755$\pm$0.9701  &  2.8593$\pm$1.0025  \\ \hline
deg8   & \textbf{0.0329}$\pm$\textbf{0.0191}  &  \textbf{0.0163}$\pm$\textbf{0.0093}  &  2.9417$\pm$0.5468  \\ \hline
deg7   & \textbf{0.0329}$\pm$\textbf{0.0191}  &  \textbf{0.0163}$\pm$\textbf{0.0093}  &  \textbf{2.4667}$\pm$\textbf{0.4254}  \\ \hline
\end{tabular}

\vspace{1mm}
\begin{tabular}{|wl{0.17\linewidth}|wc{0.21\linewidth}|wc{0.21\linewidth}|wc{0.21\linewidth}|}
\hline
 \multicolumn{4}{|c|}{KITTI sequence 07 (1100 frames)} \\ \hline
 \multirow{2}{*}{method} & rotation (°)  & translation (m) & time (ms) \\ 
 & $\mu(\delta r) \pm \sigma(\delta r)$ & $\mu(\delta t) \pm \sigma(\delta t)$ & $\mu(t) \pm \sigma(t)$ \\ \hline
Malis \cite{Malis2024}  & 1.1871$\pm$14.299  &  0.0866$\pm$0.9645 &  3.5282$\pm$0.9402  \\ \hline
Wientapper \cite{Wientapper-CVIU-2018}& 0.0467$\pm$0.0391  &  0.0121$\pm$0.07  &  49.726$\pm$16.979  \\ \hline
Zhou \cite{Zhou-ICRA-2020}  & 0971$\pm$1.0427 & 0.0136$\pm$0.0297  &  2.8219$\pm$0.6593  \\ \hline
deg8   & \textbf{0.0428}$\pm$ \textbf{0.0335}  &  \textbf{0.012}$\pm$\textbf{0.007}  &  3.498$\pm$1.0466  \\ \hline
deg7   & \textbf{0.0428}$\pm$ \textbf{0.0335}  &  \textbf{0.012}$\pm$\textbf{0.007}  &  \textbf{2.5303}$\pm$\textbf{0.7546}  \\ \hline
\end{tabular}
\caption{Comparison on KITTI datasets.}
\label{table1}
\end{table}

\subsection{Pose estimation from 3D to 2D correspondences}
The proposed methods can also be applied to the Pnp problem after derivation of the "canonical form" \eqref{eq:66}. We compare them with the following solvers:

UPnp \cite{Kneip-ECCV-2014} is an efficient solver that does not use the Lagrangian to solve the constrained least square problem. The problem solved by the UPnP is equivalent to set $\lambda = 0$ in \eqref{eq:133}. This reduces the number of solutions and hence the computational time. However, setting $\lambda = 0$ is a valid approximation only if the noise is (close to) zero.

Similarly, optDLS \cite{Nakano-optDLS} and SRPnP \cite{Wang-SRPnP} do not recover all the possible solutions. Furthermore, these methods use Cayley parameterization of the rotation (optDLS, SRPnP) do not perform well for rotations with an angle close to $\pi$ around any axis (see \cite{Zheng-OPnP}). 

SQPnP \cite{Terzakis-SQPnP} is an iterative method. Unlike closed-form methods, there is no theoretical guarantee that the method will reach the global minimizer. 

OPnP \cite{Zheng-OPnP} uses quaternion representation of the rotation and therefore is stable for rotation angles close to $\pi$ and it finds all 40 solutions. Our method finds the same number of solutions with the same accuracy, but is about 5 times faster.


\subsubsection{Experiments with simulated data}
In Table \ref{table2}, we compare the proposed methods with more solvers. It shows the average and maximal rotation and translation error and the average computation time over 10000 trials. We set the number of correspondences $N = 10$ and the noise standard deviation $\sigma = 1$ pixel. 

\begin{table*}[h]
\small
\centering
\begin{tabular}{|wc{0.15\columnwidth}|wc{0.3\columnwidth}wc{0.2\columnwidth}|wc{0.3\columnwidth}wc{0.2\columnwidth}|wc{0.4\columnwidth}|}
\hline
\multirow{2}{*}{} & \multicolumn{2}{c|}{rotation (°)}  & \multicolumn{2}{c|}{translation (m)} & time (ms)   \\ \cline{2-6} 
  & \multicolumn{1}{l|}{$\mu(\delta r) \pm \sigma(\delta r)$} & $\rm{max} (\delta r)$ & \multicolumn{1}{l|}{$\mu(\delta t) \pm \sigma(\delta t)$} & $\rm{max} (\delta t)$ & $\mu(t) \pm \sigma(t)$ \\ \hline
UPnP   & \multicolumn{1}{l|}{0.254 $\pm$ 0.630} & 33.07 & \multicolumn{1}{l|}{0.005 $\pm$ 0.007} & 0.380 & 0.734 $\pm$ 0.332 (in C) \\ \hline
optDLS & \multicolumn{1}{l|}{0.265 $\pm$ 1.604} & 71.61 & \multicolumn{1}{l|}{0.008 $\pm$ 0.099} & 5.029 & 2.000 $\pm$ 0.578  \\ \hline
OPnP   & \multicolumn{1}{l|}{0.207 $\pm$ 0.105} & 0.843 & \multicolumn{1}{l|}{0.004 $\pm$ 0.003} & 0.021 & 11.75 $\pm$ 2.747  \\ \hline
SRPnP   & \multicolumn{1}{l|}{0.214 $\pm$ 0.155} & 10.79 & \multicolumn{1}{l|}{0.004 $\pm$ 0.003} & 0.133 & 0.759 $\pm$ 0.409  \\ \hline
SQPnP   & \multicolumn{1}{l|}{0.629 $\pm$ 6.373} & 177.8 & \multicolumn{1}{l|}{0.014 $\pm$ 0.159} & 3.190 & 0.085 $\pm$ 0.546 (in C++) \\ \hline
deg9   & \multicolumn{1}{l|}{0.207 $\pm$ 0.102} & 0.844 & \multicolumn{1}{l|}{0.004 $\pm$ 0.003} & 0.021 & 2.561 $\pm$ 1.004  \\ \hline
deg8   & \multicolumn{1}{l|}{0.207 $\pm$ 0.102} & 0.844 & \multicolumn{1}{l|}{0.004 $\pm$ 0.003} & 0.021 & 2.417 $\pm$ 0.856  \\ \hline
deg7   & \multicolumn{1}{l|}{0.207 $\pm$ 0.102} & 0.844 & \multicolumn{1}{l|}{0.004 $\pm$ 0.003} & 0.021 & 2.204 $\pm$ 0.768  \\ \hline
\end{tabular}
\caption{Comparison to other PnP solvers.}
\label{table2}
\end{table*}

To further test the accuracy of the proposed methods, we used images from the ETH3D dataset \cite{ETH3D-2017}. The dataset contains original images, the coordinates of 3D points and the corresponding 2D points, the pose of cameras associated with each image and the intrinsic parameters of the cameras. We compare the proposed methods with the algorithm from \cite{Malis2024} and the UPnp method \cite{Kneip-ECCV-2014} on the 25 available datasets. 
Figure \ref{fig:3Dto2D} shows the mean rotation and translation error averaged across all datasets. In each dataset, we chose one image and 10 correspondences, where we perturbed the 2D and 3D ground truth points by increasing noise. Methods solving the least squares problem exactly, including ours, show better accuracy with increasing noise. The proposed approaches were implemented in Matlab and are slower than the UPnp method. Indeed, we used the UPnp implementation in C available in the OpenGV library \cite{Kneip-OpenGV-2014}. However, more accurate approaches may be useful in non-real time application.

\begin{figure}[t]
\captionsetup[subfigure]{labelformat=empty}
\centering
\begin{subfigure}[b]{0.5\linewidth}
\centering
\includegraphics[width=0.95\linewidth]{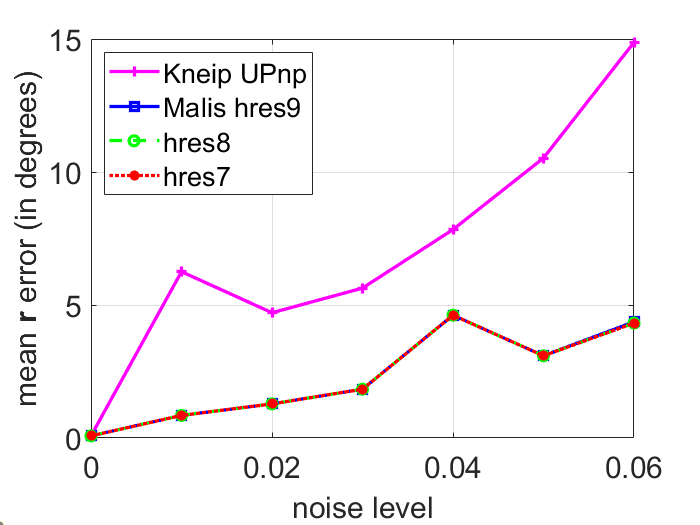}
\caption{Rotation error (degrees).}
\end{subfigure}%
\begin{subfigure}[b]{0.5\linewidth}
\centering
\includegraphics[width=0.95\linewidth]{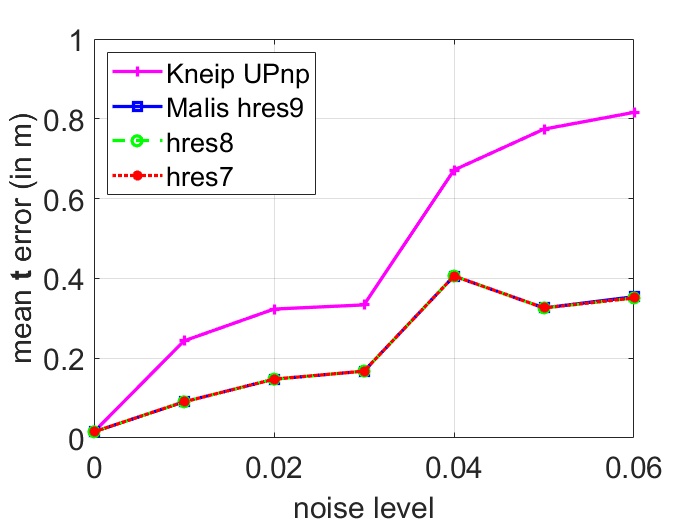}
\caption{Translation error (m).}
\end{subfigure}
\caption{Pnp problem on ETH3D dataset with increasing noise.}
\label{fig:3Dto2D}
\end{figure}

\subsubsection{Experiments with real data}
To test our methods on real data, we again used the ETH3D dataset \cite{ETH3D-2017}. The dataset contains 25 sequences of images, together with the coordinates of 3D points and the corresponding 2D points, the pose of cameras associated with each image, and the intrinsic parameters of the cameras. For every sequence, we chose one reference image $\textbf{im}^i_0$, where $i = 1,\dots,25$. In every other image $\textbf{im}^i_j$ of the i-th sequence, we matched points with points in the reference image. The preimages of the matched points in the reference image $\textbf{im}^i_0$ are the 3D points $\f m_r$ from equation \eqref{eq:79}. In image $\textbf{im}^i_j$, we filter out matched points, that are farther than 5 pixels from the true matches. The remaining points are the image points $\f p_c$ in the image $\textbf{im}_j^i$. We estimated the pose $\hat{\f R}, \hat{\f t}$, and compared it with the ground truth pose $\f R,\f t$, such that equation \eqref{eq:79} holds.

We did not consider images $\textbf{im}^i_j$ that do not overlap with the reference image $\textbf{im}^i_0$, nor those with fewer than 4 correspondences.

\begin{table}[t]
\centering
\footnotesize
\begin{tabular}{|wl{0.205\columnwidth}|wc{0.31\columnwidth}|wc{0.31\columnwidth}|}
\hline 
\multirow{2}{*}{} & rotation (°)&translation (m)\\ \cline{2-3} 
    & $\mu(\delta r) \pm \sigma(\delta r)$ &$\mu(\delta t) \pm \sigma(\delta t)$\\ \hline
Kneip UPnp & 6.9055 $\pm$ 18.2656 & 0.80875 $\pm$ 2.8761 \\ \hline
Malis deg9 & 2.1398 $\pm$ 5.9209 & 0.44969 $\pm$ 1.2829 \\ \hline
deg8 & 2.2221 $\pm$ 6.3222 & \textbf{0.4363 $\pm$ 1.2553}  \\ \hline
deg7 & \textbf{2.1263 $\pm$ 5.8971} & 0.46607 $\pm$ 1.3665  \\ \hline
\end{tabular}
\caption{Comparison of the developed methods with UPnp method by Kneip, Li and Seo \cite{Kneip-ECCV-2014}, and the deg9 method by Malis \cite{Malis2024}. The results are averaged over all sequences in the ETH3D dataset.}
\label{table3}
\end{table}

\begin{table}[t]
\centering
\footnotesize
\begin{tabular}{|wl{0.12\columnwidth}|wc{0.15\columnwidth}|wc{0.15\columnwidth}|wc{0.15\columnwidth}|wc{0.15\columnwidth}|}
\hline 
 n & Kneip UPnp & Malis deg9 & deg8 & deg7 \\ \hline
4 & 35.1584 & 20.0303 & 18.4840 & \textbf{15.3291}  \\ \hline
5 & 45.7671 & 21.0113 & 20.7705 & \textbf{18.9686} \\ \hline
6 & 53.4829 & 14.4955 & 14.4299 & \textbf{14.0446} \\ \hline
7 & 44.3037 & \textbf{10.0844} & 11.0800 & 10.1060 \\ \hline
8 & 28.4090 & \textbf{13.0101} & 13.0107 & \textbf{13.0101} \\ \hline
9 & 24.9650 & 7.8337 & 7.8338 & \textbf{7.8331} \\ \hline
10 & 2.2087 & 0.7268 & \textbf{0.7267} & 0.7268 \\ \hline
11 - 25 & 12.4584 & \textbf{4.0925} & \textbf{4.0925} & \textbf{4.0925} \\ \hline
26 - 50 & 10.7784 & \textbf{2.7428} & \textbf{2.7428} & \textbf{2.7428} \\ \hline
51 - 100 & 11.5497 & \textbf{1.9820} & \textbf{1.9820} & \textbf{1.9820} \\ \hline
101 - 200 & 4.3847 &  \textbf{0.3016} & \textbf{0.3016} & \textbf{0.3016} \\ \hline
$\geq$ 200 & 8.7278 & \textbf{0.5582} & \textbf{0.5582} & \textbf{0.5582} \\ \hline
\end{tabular}
\caption{Mean rotational error (in degrees) for given number n of correspondences. The results are averaged over all sequences in the ETH3D dataset.}
\label{table4}
\end{table}

\begin{table}[t]
\centering
\footnotesize
\begin{tabular}{|wl{0.12\columnwidth}|wc{0.15\columnwidth}|wc{0.15\columnwidth}|wc{0.15\columnwidth}|wc{0.15\columnwidth}|}
\hline 
n & Kneip UPnp & Malis deg9 & deg8 & deg7 \\ \hline
4 & 8.2631 & 6.6374 & 6.8478 & \textbf{3.7874} \\ \hline
5 & 1.5255 & 3.5584 & 3.3815 & \textbf{2.7821} \\ \hline
6 & 11.5245 & 2.4866 & 2.4848 & \textbf{2.4597} \\ \hline
7 & 1.4194 & \textbf{2.0547} & 2.1253 & 2.0578 \\ \hline
8 & 0.6908 & 1.4533 & \textbf{1.4531} & 1.4533 \\ \hline
9 & 5.8250 & \textbf{1.9149} & \textbf{1.9149} & \textbf{1.9149} \\ \hline
10 & 0.6147 & \textbf{0.1549} & \textbf{0.1549} & \textbf{0.1549} \\ \hline
11 - 25 & 12.4584 & \textbf{4.0925} & \textbf{4.0925} & \textbf{4.0925} \\ \hline
26 - 50 & 1.0293 & \textbf{0.6142} & \textbf{0.6142} & \textbf{0.6142} \\ \hline
51 - 100 & 0.8434 & \textbf{0.1951} & \textbf{0.1951} & \textbf{0.1951} \\ \hline
101 - 200 & 4.3847 & \textbf{0.3016} & \textbf{0.3016} & \textbf{0.3016} \\ \hline
$\geq$ 200 & 0.6028 & \textbf{0.0429} & \textbf{0.0429} & \textbf{0.0429} \\ \hline
\end{tabular}
\caption{Mean translation error (in meters) for given number n of correspondences. The results are averaged over all sequences in the ETH3D dataset.}
\label{table5}
\end{table}

We compared the proposed methods with the algorithm by Malis \cite{Malis2024} and the UPnp method by Kneip, Li and Seo \cite{Kneip-ECCV-2014} (without the final Newton step). Table \ref{table3} shows the rotation and translation error averaged over the 25 sequences.
Tables \ref{table4} and \ref{table5} show the average rotation and translation error (respectively) for a given number of correspondences. The results are again averaged over the 25 sequences. All the compared algorithms improve with growing number of correspondences. The resultant-based methods outperform the state-of-the-art UPnp method and, for low number of correspondences, the proposed deg8 and deg7 methods  often perform better than the deg9 algorithm by Malis.

\section{Conclusion}
In this work, we showed how to integrate Sylvester forms in resultant-based methods
and proved the validity of our approach. 
We obtained new resultant-based methods that operate in degrees 7 and 8, significantly reducing the size of the elimination matrices. This has a significant impact on the computation time outperforming previous approaches.

An important open question concerns the selection of the monomial ordering, which affects the conditioning of the blocks of the elimination matrix from which the solutions are computed. Fixing the block construction in advance is computationally more efficient, whereas selecting it online can improve numerical accuracy. Understanding how to choose this ordering optimally remains an interesting direction for future work.


\bibliographystyle{IEEEtran}
\bibliography{main}

\end{document}